\documentclass{article}

\usepackage{amsmath}
\usepackage{amsfonts}
\usepackage{amsthm}
\usepackage{multirow}
\usepackage{url}

\usepackage{algorithm}
\usepackage{algorithmic}

\usepackage{bm}
\usepackage{graphics}
\usepackage{graphicx}
\usepackage{threeparttable}

\newcounter{counter_lemma}

\newtheorem{signPreserving}[counter_lemma]{Lemma}
\newtheorem{ExplictSparsity}[counter_lemma]{Lemma}

\newtheorem{MinMaxOmega}[counter_lemma]{Lemma}
\newtheorem{convexLip}[counter_lemma]{Lemma}

\newtheorem{InducedSparsity}[counter_lemma]{Lemma}

\newcounter{counter_theorem}

\newtheorem{ReducedProblem}[counter_theorem]{Theorem}
\newtheorem{omegaSmooth}[counter_theorem]{Theorem}
\newtheorem{dualityGap}[counter_theorem]{Theorem}

\title{Fast Overlapping Group Lasso}

\author{
Jun Liu and Jieping Ye\\
Department of CSE \\
Arizona State University \\
Tempe, AZ 85287, USA}

\date{August 29, 2010}

\begin{document}

\maketitle

\begin{abstract}

The group Lasso is an extension of the Lasso for feature selection on
(predefined) non-overlapping groups of features. The non-overlapping
group structure limits its applicability in practice. There have been
several recent attempts to study a more general formulation, where
groups of features are given, potentially with overlaps between the
groups. The resulting optimization is, however, much more challenging
to solve due to the group overlaps. In this paper, we consider the
efficient optimization of the overlapping group Lasso penalized
problem. We reveal several key properties of the proximal operator
associated with the overlapping group Lasso, and compute the proximal
operator by solving the smooth and convex dual problem, which allows
the use of the gradient descent type of algorithms for the
optimization. We have performed empirical evaluations using the
breast cancer gene expression data set, which consists of 8,141 genes
organized into (overlapping) gene sets. Experimental results
demonstrate the efficiency and effectiveness of the proposed
algorithm.

\end{abstract}

\section{Introduction}

Problems with high dimensionality have become common over the recent
years. The high dimensionality poses significant challenges in
building interpretable models with high prediction accuracy.
Regularization has been commonly employed to obtain more stable and
interpretable models. A well-known example is the penalization of
the $\ell_1$ norm of the estimator, known as
Lasso~\cite{TIBSHIRANI:1996:ID24}. The $\ell_1$ norm regularization
has achieved great success in many applications. However, in some
applications~\cite{Yuan:2006}, we are interested in finding
important explanatory factors in predicting the response variable,
where each explanatory factor is represented by a group of input
features. In such cases, the selection of important features
corresponds to the selection of groups of features. As an extension
of Lasso, group Lasso~\cite{Yuan:2006} based on the combination of
the $\ell_1$ norm and the $\ell_2$ norm has been proposed for group
feature selection, and quite a few efficient
algorithms~\cite{Liu:han:2009:blockwise:cd,Liu:2009:UAI,Meier:2008}
have been proposed for efficient optimization. However, the
non-overlapping group structure in group Lasso limits its
applicability in practice. For example, in microarray gene
expression data analysis, genes may form overlapping groups as each
gene may participate in multiple pathways~\cite{Jacob:2009}.

Several recent
work~\cite{Bondell:2008,Jacob:2009,Kim:20010:icml,Zhao:2009} studies
the overlapping group Lasso, where groups of features are given,
potentially with overlaps between the groups. The resulting
optimization is, however, much more challenging to solve due to the
group overlaps. When optimizing the overlapping group Lasso problem,
one can reformulate it as a second order cone program and solve it
by the generic toolboxes, which, however, does not scale well.
In~\cite{Jenatton:2009}, an alternating algorithm called SLasso is
proposed for solving the equivalent reformulation. However, SLasso
involves an expensive matrix inversion at each alternating
iteration, and there is no known global convergence rate for such an
alternating procedure. It was recently shown
in~\cite{Jenatton:2010:icml} that, for the tree structured group
Lasso, the associated proximal operator (or equivalently, the
Moreau-Yosida reguralization)~\cite{Moreau:1965,Yosida:1964} can be
computed by applying block coordinate ascent in the dual and the
algorithm converges in one pass. It was shown independently
in~\cite{Liu:nips:2010} that the proximal operator associated with
the tree structured group Lasso has a nice analytical solution.
However, to the best of our knowledge, there is no analytical
solution to the proximal operator associated with the general
overlapping group Lasso.

In this paper, we develop an efficient algorithm for the overlapping
group Lasso penalized problem via the accelerated gradient descent
method. The accelerated gradient descent method has recently received
increasing attention in machine learning due to the fast convergence
rate even for nonsmooth convex problems. One of the key operations is
the computation of the proximal operator associated with the penalty.
We reveal several key properties of the proximal operator associated
with the overlapping group Lasso penalty, and compute the proximal
operator by solving the dual problem. The main contributions of this
paper include: (1) we develop a procedure to identify many zero
groups in the proximal operator, which dramatically reduces the size
of the dual problem to be solved; (2) we show that the dual problem
is smooth and convex with Lipschitz continuous gradient, thus can be
solved by existing smooth convex optimization tools; and (3) we
derive the duality gap between the primal and dual problems, which
can be used to check the quality of the solution and determine the
convergence of the algorithm. We have performed empirical evaluations
using the breast cancer gene expression data set, which consists of
8,141 genes organized into (overlapping) gene sets. Experimental
results demonstrate the efficiency and effectiveness of the proposed
algorithm.

\textbf{Notations: }  $\|\cdot\|$ denotes the Euclidean norm, and
$\mathbf 0$ denotes a vector of zeros. ${\rm SGN}(\cdot)$ and ${\rm
sgn} (\cdot)$ are defined in a componentwise fashion as: 1) if $t=0$,
then ${\rm SGN}(t)=[-1, 1]$ and ${\rm sgn}(t)=0$; 2) if $t>0$, then
${\rm SGN}(t)=\{1\}$ and ${\rm sgn}(t)=1$; and 3) if $t<0$, ${\rm
SGN}(t)=\{-1\}$ and ${\rm sgn}(t)=-1$. $G_i \subseteq \{1, 2, \ldots,
p\}$ denotes an index set, and $\mathbf x_{G_i}$ denote a subvector
of $\mathbf x$ consisting of the entries indexed by $G_i$.

\section{The Overlapping Group Lasso}\label{s:problem:statement}

We consider the following overlapping group Lasso penalized problem:
\begin{equation}\label{eq:problem}
    \min_{\mathbf x \in \mathbb{R}^p} f(\mathbf x)=l(\mathbf x) +
      \phi_{\lambda_2}^{\lambda_1}(\mathbf x)
\end{equation}
where $l(\cdot)$ is a smooth convex loss function, e.g., the least
squares loss,
\begin{equation}\label{eq:phi}
 \phi_{\lambda_2}^{\lambda_1}(\mathbf x)=\lambda_1 \|\mathbf x\|_1 +
 \lambda_2 \sum_{i=1}^g w_i \|\mathbf x_{G_i}\|
\end{equation}
is the overlapping group Lasso penalty, $\lambda_1 \geq 0$ and
$\lambda_2 \geq 0$ are regularization parameters, $w_i > 0, i=1, 2,
\ldots, g$, $G_i \subseteq \{1, 2, \ldots, p\}$ contains the indices
corresponding to the $i$-th group of features, and $\|\cdot\|$
denotes the Euclidean norm. The $g$ groups of features are
pre-specified, and they may overlap. The penalty in \eqref{eq:phi} is
a special case of the more general Composite Absolute Penalty (CAP)
family~\cite{Zhao:2009}. When the groups are disjoint with
$\lambda_1=0$ and $\lambda_2>0$, the model in \eqref{eq:problem}
reduces to the group Lasso~\cite{Yuan:2006}. If $\lambda_1>0$ and
$\lambda_2=0$, then the model in \eqref{eq:problem} reduces to the
standard Lasso~\cite{TIBSHIRANI:1996:ID24}.


In this paper, we propose to make use of the accelerated gradient
descent (AGD)~\cite{Beck:2009,NEMIROVSKI:1994:ID6,Nesterov:2004} for
solving~\eqref{eq:problem}, due to its fast convergence rate. The
algorithm is called ``FoGLasso", which stands for \textbf{F}ast
\textbf{o}verlapping \textbf{G}roup \textbf{Lasso}. One of the key
steps in the proposed FoGLasso algorithm is the computation of the
proximal operator associated with the penalty in~\eqref{eq:phi}; and
we present an efficient algorithm for the computation in the next
section.


In FoGLasso, we first construct a model for approximating $f(\cdot)$
at the point $\mathbf x$ as:
\begin{equation}\label{eq:model:mL}
    f_{L, \mathbf x} (\mathbf y)  =
    [l(\mathbf x) + \langle l'(\mathbf x), \mathbf y - \mathbf x  \rangle]
    +  \phi_{\lambda_1}^{\lambda_2}(\mathbf y)+ \frac{L}{2} \|\mathbf y - \mathbf x\|^2,
\end{equation}
where $L >0$. The model $f_{L, \mathbf x}  (\mathbf y)$ consists of
the first-order Taylor expansion of the smooth function $l(\cdot)$ at
the point $\mathbf x$, the non-smooth penalty
$\phi_{\lambda_1}^{\lambda_2}(\mathbf x)$, and a regularization term
$\frac{L}{2} \|\mathbf y - \mathbf x\|^2$. Next, a sequence of
approximate solutions $\{\mathbf x_i\}$ is computed as follows:
\begin{equation}\label{eq:xkplus1}
    \mathbf x_{i+1}= \arg \min_{\mathbf y} f_{L_i, \mathbf s_i} (\mathbf y)
\end{equation}
where the search point $\mathbf s_i$ is an affine combination of $\mathbf
x_{i-1}$ and $\mathbf x_i$ as
\begin{equation}\label{eq:sk}
    \mathbf s_i= \mathbf x_i + \beta_i (\mathbf x_i - \mathbf x_{i-1}),
\end{equation}
for a properly chosen coefficient $\beta_i$, $L_i$ is determined by
the line search according to the Armijo-Goldstein rule so that $L_i$
should be appropriate for $\mathbf s_i$, i.e., $f( \mathbf x_{i+1})
\leq f_{L,\mathbf s_i}( \mathbf x_{i+1})$.  Following the analysis
in~\cite{Beck:2009,Nesterov:2004}, we can show that FoGLasso achieves
a convergence rate of $O(1/k^2)$ for $k$ iterations, which is optimal
among first-order methods. A key building block in FoGLasso is the
minimization of \eqref{eq:model:mL}, whose solution is known as the
proximal operator~\cite{Hiriart-Urruty:1993,Moreau:1965}. The
computation of the proximal operator is the main technical
contribution of this paper. The pseudo-code of FoGLasso is summarized
in Algorithm~\ref{algorithm:nest}, where the proximal operator
$\pi(\cdot)$ is defined in \eqref{eq:problem:m:y}. In practice, we
can terminate Algorithm~\ref{algorithm:nest} if the change of the
function values corresponding to adjacent iterations is within a
small value, say $10^{-5}$.

\begin{algorithm} [ht]
  \caption{The FoGLasso Algorithm}
  \label{algorithm:nest}
\begin{algorithmic}[1]
  \REQUIRE $ L_0 >0, \mathbf x_0, k$
  \ENSURE  $\mathbf x_{k+1}$
    \STATE Initialize $\mathbf x_1=\mathbf x_0$, $\alpha_{-1}=0$, $\alpha_0=1$, and $L=L_0$.
    \FOR{$i=1$ to $k$}
     \STATE Set $\beta_i= \frac{\alpha_{i-2}-1}{\alpha_{i-1}}$, $\mathbf s_i =
     \mathbf x_i + \beta_i (\mathbf x_i  - \mathbf x_{i-1})$
     \STATE Find the smallest $L=2^j L_{i-1}, j=0, 1, \ldots$ such that
 $   f( \mathbf x_{i+1}) \leq f_{L,\mathbf s_i}( \mathbf
x_{i+1})$ holds,
where $\mathbf x_{i+1}=\pi_{\lambda_2/L}^{\lambda_1/L} (\mathbf s_i -
\frac{1}{L} l'(\mathbf s_i))$
    \STATE Set $L_i=L$ and $\alpha_{i+1}=\frac{1+\sqrt{1+4  \alpha_i^2}}{2}$
    \ENDFOR
\end{algorithmic}
\end{algorithm}


\section{The Associated Proximal Operator and Its Efficient
Computation}\label{s:proximal:operator}

The proximal operator associated with the overlapping group Lasso
penalty is defined as follows:
\begin{equation}\label{eq:problem:m:y}
    \pi_{\lambda_2}^{\lambda_1}(\mathbf v)= \arg \min_{\mathbf x \in \mathbb{R}^p}
    \left\{ g_{\lambda_2}^{\lambda_1}(\mathbf x) \equiv \frac{1}{2}\|\mathbf x - \mathbf v\|^2
    + \lambda_1 \|\mathbf x\|_1 +  \lambda_2 \sum_{i=1}^g w_i \|\mathbf
    x_{G_i}\|\right\},
\end{equation}
which is a special case of~\eqref{eq:problem} by setting $l(\mathbf
x)=\frac{1}{2}\|\mathbf x - \mathbf v\|^2$. It can be verified that
the approximate solution $x_{i+1}$ in \eqref{eq:xkplus1} is given by
$\mathbf x_{i+1}=\pi_{\lambda_2/L_i}^{\lambda_1/L_i} (\mathbf s_i -
\frac{1}{L_i} l'(\mathbf s_i))$. Recently, it has been shown
in~\cite{Jenatton:2010:icml,Liu:nips:2010,Liu:kdd:2010} that, the
efficient computation of the proximal operator is key to many sparse
learning algorithms~\cite[Section 2]{Liu:kdd:2010}. Next, we focus
on the efficient computation of $\pi_{\lambda_2}^{\lambda_1}(\mathbf
v)$ in \eqref{eq:problem:m:y} for a given $\mathbf v$.

\subsection{Key Properties of the Proximal Operator}\label{ss:proximal:operator}

We first reveal several basic properties of the proximal operator
$\pi_{\lambda_2}^{\lambda_1}(\mathbf v)$.

\begin{signPreserving}\label{lemma:signPreserving}
Suppose that $\lambda_1, \lambda_2 \geq 0$, and $w_i > 0$, for $i=1,
2, \ldots, g$. Let $\mathbf x^*=\pi_{\lambda_2}^{\lambda_1}(\mathbf
v)$. The following holds: 1) if $v_i >0$, then $0 \leq x_i^* \leq
v_i$; 2) if $v_i <0$, then $v_i \leq x_i^* \leq 0$; 3) if $v_i =0$,
then $x_i^* = 0$; 4) ${\rm SGN}(\mathbf v) \subseteq {\rm
SGN}(\mathbf x^*)$; and 5) $\pi_{\lambda_2}^{\lambda_1}(\mathbf
v)={\rm sgn}(\mathbf v) \odot \pi_{\lambda_2}^{\lambda_1}(|\mathbf
v|)$.
\end{signPreserving}

\begin{proof}
When $\lambda_1, \lambda_2 \geq 0$, and $w_i \geq 0$, for $i=1, 2,
\ldots, g$, the objective function $g_{\lambda_2}^{\lambda_1}(\cdot)$
is strictly convex, thus $\mathbf x^*$ is the unique minimizer. We
first show if $v_i >0$, then $0 \le x_i^* \le v_i$.  If $x_i^* >
v_i$, then we can construct a $\mathbf {\hat x}$ as follows: $\hat
x_j=x_j^*, j\neq i$ and $\hat x_i= v_i$. Similarly, if $x_i^* <0$,
then we can construct a $\mathbf {\hat x}$ as follows: $\hat
x_j=x_j^*, j\neq i$ and $\hat x_i= 0$. It is easy to verify that
$\mathbf {\hat x}$ achieves a lower objective function value than
$\mathbf x^*$ in both cases. We can prove the second and the third
properties using similar arguments. Finally, we can prove the fourth
and the fifth properties using the definition of $\mbox{SGN}(\cdot)$
and the first three properties.
\end{proof}

Next, we show that $\pi_{\lambda_2}^{\lambda_1}(\cdot)$ can be
directly derived from $\pi_{\lambda_2}^{0}(\cdot)$ by
soft-thresholding. Thus, we only need to focus on the simple case
when $\lambda_1=0$. This simplifies the optimization in
\eqref{eq:problem:m:y}.

\begin{ReducedProblem}\label{theorem:reducedProblem}
Let
\begin{equation}\label{eq:u:v}
    \mathbf u= {\rm sgn}(\mathbf v) \odot \max(|\mathbf v| -
\lambda_1, 0),
\end{equation}
\begin{equation}\label{eq:h:x}
    \pi_{\lambda_2}^{0}(\mathbf u) = \arg \min_{\mathbf x \in \mathbb{R}^p}
    \left \{h_{\lambda_2}(\mathbf x) \equiv \frac{1}{2}\|\mathbf x - \mathbf u\|^2
    +  \lambda_2 \sum_{i=1}^g w_i \|\mathbf   x_{G_i}\| \right \}.
\end{equation}
The following holds:
\begin{equation}\label{eq:problem:m:y:reduced:relationship}
    \pi_{\lambda_2}^{\lambda_1}(\mathbf v)=\pi_{\lambda_2}^{0}(\mathbf u).
\end{equation}
\end{ReducedProblem}

\begin{proof}
Denote the unique minimizer of $h_{\lambda_2}(\cdot)$ as $\mathbf
x^*$. The sufficient and necessary condition for the optimality of
$\mathbf x^*$ is:
\begin{equation}\label{eq:h:subdifferential}
   \mathbf 0 \in \partial h_{\lambda_2}(\mathbf x^*) = \mathbf x^* - \mathbf u +
    \partial \phi_{\lambda_2}^{0} (\mathbf x^*),
\end{equation}
where $\partial h_{\lambda_2}(\mathbf x)$ and $\partial
\phi_{\lambda_2}^{0} (\mathbf x)$ are the subdifferential sets of
$h_{\lambda_2}(\cdot)$ and $\phi_{\lambda_2}^{0}(\cdot)$ at $\mathbf
x$, respectively.

To prove~\eqref{eq:problem:m:y:reduced:relationship}, it suffices to
show $\mathbf 0 \in \partial g_{\lambda_2}^{\lambda_1}(\mathbf x^*)$.
The subdifferential of $g_{\lambda_2}^{\lambda_1}(\cdot)$ at $\mathbf
x^*$ is given by
\begin{equation}\label{eq:g:subdifferential}
   \partial g_{\lambda_2}^{\lambda_1}(\mathbf x^*)
   = \mathbf x^* - \mathbf v + \partial \phi_{\lambda_2}^{\lambda_1} (\mathbf x^*)
   =\mathbf x^* - \mathbf v + \lambda_1 {\rm SGN}(\mathbf x^*) +
   \partial \phi_{\lambda_2}^{0} (\mathbf x^*).
\end{equation}
It follows from~\eqref{eq:u:v} that
$   \mathbf u \in \mathbf v - \lambda_1 \mbox{SGN}(\mathbf u)$.
Using the fourth property in Lemma~\ref{lemma:signPreserving}, we
have
$\mbox{SGN}(\mathbf u) \subseteq {\rm SGN}(\mathbf x^*)$.
Thus,
\begin{equation}\label{eq:u:v:x:SGN}
   \mathbf u \in \mathbf v - \lambda_1 {\rm SGN}(\mathbf x^*).
\end{equation}
It follows from \eqref{eq:h:subdifferential}-\eqref{eq:u:v:x:SGN}
that $\mathbf 0 \in \partial g_{\lambda_2}^{\lambda_1}(\mathbf x^*)$.
\end{proof}

It follows from Theorem~\ref{theorem:reducedProblem} that, we only
need to focus on the optimization of \eqref{eq:h:x} in the following
discussion. The difficulty in the optimization of \eqref{eq:h:x}
lies in the large number of groups that may overlap. In practice,
many groups will be zero, thus achieving a sparse
solution\footnote{The sparse solution is much more desirable than
the non-sparse one in many applications. For the non-sparse case,
one may apply the subgradient based methods such as those proposed
in~\cite{Shalev-Shwartz:2007:icml,Xiao:average:2009} for
solving~\eqref{eq:h:x}, which deserves further study.}. However, the
zero groups are not known in advance. The key question we aim to
address is how we can identify as many zero groups as possible to
reduce the complexity of the optimization. We present in the next
lemma a sufficient condition for a group to be zero.

\begin{ExplictSparsity}\label{lemma:ExplictSparsity}
Denote the minimizer of $h_{\lambda_2}(\cdot)$ in \eqref{eq:h:x} by
$\mathbf x^*$. If the $i$-th group satisfies $\|\mathbf u_{G_i}\|
\leq \lambda_2 w_i$, then $\mathbf x_{G_i}^*=\mathbf 0$, i.e., the
$i$-th group is zero.
\end{ExplictSparsity}

\begin{proof}
We decompose $h_{\lambda_2}(\mathbf x)$ into two parts as follows:
\begin{equation}\label{eq:h:x:rewrite}
    h_{\lambda_2}(\mathbf x) = \left\{ \frac{1}{2}\|\mathbf x_{G_i} -
    \mathbf u_{G_i}\|^2
    +  \lambda_2 w_i \|\mathbf   x_{G_i}\| \right\} +
    \left\{ \frac{1}{2}\|\mathbf x_{\overline{G}_i} -
    \mathbf u_{\overline{G}_i}\|^2 +
    \lambda_2 \sum_{j \neq i} w_j \|\mathbf x_{G_j}\| \right\},
\end{equation}
where $\overline{G}_i=\{1,2, \ldots, p\} - G_i$ is the complementary
set of $G_i$. We consider the minimization of $h_{\lambda_2}(\mathbf
x)$ in terms of $\mathbf x_{G_i}$ when $\mathbf x_{\overline{G}_i} =\mathbf
x^*_{\overline{G}_i}$ is fixed. It can be verified that if $\|\mathbf
u_{G_i}\| \leq \lambda_2 w_i$, then $\mathbf x_{G_i}^*=\mathbf 0$
minimizes both terms in~\eqref{eq:h:x:rewrite} simultaneously. Thus
we have $\mathbf x_{G_i}^*=\mathbf 0$.
\end{proof}

Lemma~\ref{lemma:ExplictSparsity} may not identify many true zero
groups due to the strong condition imposed. The lemma below weakens
the condition in Lemma~\ref{lemma:ExplictSparsity}. Intuitively, for
a group $G_i$, we first identify all existing zero groups that
overlap with $G_i$, and then compute the overlapping index subset
$S_i$ of $G_i$ as:
\begin{equation}\label{eq:set:S:i}
S_i=\bigcup_{j \neq i, \mathbf x^*_{G_j}=\mathbf 0} (G_j \cap G_i).
\end{equation}
We can show that $\mathbf x^*_{G_i}=\mathbf 0$ if $\|\mathbf u_{G_i -
S_i}\| \leq \lambda_2 w_i$ is satisfied. Note that this condition is
much weaker than the condition in Lemma~\ref{lemma:ExplictSparsity},
which requires that $\|\mathbf u_{G_i}\| \leq \lambda_2 w_i$.

\begin{InducedSparsity}\label{corollary:InducedSparsity}
Denote the minimizer of $h_{\lambda_2}(\cdot)$ by $\mathbf x^*$. Let
$S_i$, a subset of $G_i$, be defined in \eqref{eq:set:S:i}. If
$\|\mathbf u_{G_i - S_i}\| \leq \lambda_2 w_i$ holds, then $\mathbf
x^*_{G_i}=\mathbf 0$.
\end{InducedSparsity}
The proof of Lemma~\ref{corollary:InducedSparsity} follows similar
arguments in Lemma~\ref{lemma:ExplictSparsity} and is omitted.
Lemma~\ref{corollary:InducedSparsity} naturally leads to an iterative
procedure for identifying the zero groups: For each group $G_i$, if
$\|\mathbf u_{G_i}\| \leq \lambda_2 w_i$, then we set $\mathbf
u_{G_i}=\mathbf 0$; we cycle through all groups repeatedly until
$\mathbf u$ does not change.

%
%
%
%
%
%

Let $p'=|\{u_i: u_i \neq 0\}|$ be the number of nonzero elements in
$\mathbf u$, $g'=|\{\mathbf u_{G_i}: \mathbf u_{G_i} \neq \mathbf
0\}|$ be the number of the nonzero groups, and $\mathbf x^*$ denote
the minimizer of $h_{\lambda_2}(\cdot)$. It follows from
Lemma~\ref{corollary:InducedSparsity} and
Lemma~\ref{lemma:signPreserving} that, if $u_i =0$, then $x^*_i=0$.
Therefore, by applying the above iterative procedure, we can find the
minimizer of \eqref{eq:h:x} by solving a reduced problem that has $p'
\leq p$ variables and $g' \leq g$ groups. With some abuse of
notation, we still use \eqref{eq:h:x} to denote the resulting reduced
problem. In addition, from Lemma~\ref{lemma:signPreserving}, we only
focus on $\mathbf u>0$ in the following discussion, and the analysis
can be easily generalized to the general $\mathbf u$.

\subsection{Reformulation as an Equivalent Smooth Convex Optimization Problem}

It follows from the first property of
Lemma~\ref{lemma:signPreserving} that, we can rewrite \eqref{eq:h:x}
as:
\begin{equation}\label{eq:h:x:positive}
    \pi_{\lambda_2}^{0}(\mathbf u) = \arg \min_{\mathbf x \in \mathbb{R}^p: \mathbf x \geq \mathbf 0}
    \left \{h_{\lambda_2}(\mathbf x) \equiv \frac{1}{2}\|\mathbf x - \mathbf u\|^2
    +  \lambda_2 \sum_{i=1}^g w_i \|\mathbf   x_{G_i}\| \right \},
\end{equation}
where the minimizer of $h_{\lambda_2}(\cdot)$ is constrained to be
non-negative due to $\mathbf u>0$.

Making use of the dual norm of the Euclidean norm $\|\cdot\|$, we
can rewrite $h_{\lambda_2}(\mathbf x)$ as:
\begin{equation}\label{eq:h:x:max}
    h_{\lambda_2}(\mathbf x)= \max_{Y \in \Omega}
     \frac{1}{2}\|\mathbf x - \mathbf u\|^2 +
    \sum_{i=1}^g
    \langle \mathbf x, Y^i  \rangle,
\end{equation}
where $\Omega$ is defined as follows:
$$\Omega=\{Y \in \mathbb{R}^{p \times g}: Y^i_{\overline{G}_i}
=\mathbf 0, \|Y^i\|\leq \lambda_2 w_i, i=1, 2, \ldots, g\},$$
$\overline{G}_i$ is the complementary set of $G_i$, $Y$ is a sparse
matrix satisfying $Y_{ij}=0$ if the $i$-th feature does not belong to
the $j$-th group, i.e., $i \not \in G_{j}$, and $Y^i$ denotes the
$i$-th column of $Y$. As a result, we can reformulate
\eqref{eq:h:x:positive} as the following min-max problem:
\begin{equation}\label{eq:min:max}
    \min_{\mathbf x \in \mathbb{R}^p: \mathbf x \geq \mathbf 0} \mbox{ } \max_{ Y \in \Omega}
     \left \{ \psi(\mathbf x, Y)=\frac{1}{2}\|\mathbf x - \mathbf u\|^2 +
     \langle \mathbf x, Y \mathbf e \rangle \right \},
\end{equation}
where $\mathbf e \in \mathbb{R}^g$ is a vector of ones. It is easy to
verify that $\psi(\mathbf x, Y)$ is convex in $\mathbf x$ and concave
in $Y$, and the constraint sets are closed convex for both $\mathbf
x$ and $Y$. Thus, \eqref{eq:min:max} has a saddle point, and the
min-max can be exchanged.

It is easy to verify that for a given $Y$,  the optimal $ \mathbf x$
minimizing $\psi(\mathbf x, Y)$ in \eqref{eq:min:max} is given by
\begin{equation}\label{eq:x:u:Y:w}
   \mathbf x = \max( \mathbf u - Y \mathbf e, \mathbf 0).
\end{equation}
Plugging \eqref{eq:x:u:Y:w} into \eqref{eq:min:max}, we obtain the
following minimization problem with regard to $Y$:
\begin{equation}\label{eq:omega:Y}
    \min_{Y \in \mathbb{R}^{p \times g}: Y \in \Omega}
     \left \{ \omega( Y)=-  \psi(\max( \mathbf u - Y \mathbf e, \mathbf 0), Y)  \right \}.
\end{equation}

Our methodology for minimizing $h_{\lambda_2}(\cdot)$ defined in
\eqref{eq:h:x} is to first solve \eqref{eq:omega:Y}, and then
construct the solution to $h_{\lambda_2}(\cdot)$ via
\eqref{eq:x:u:Y:w}. We show in Theorem~\ref{theorem:omegaSmooth}
below that the function $\omega(\cdot)$ is continuously
differentiable with Lipschitz continuous gradient. Therefore, we
convert the non-smooth problem \eqref{eq:h:x:positive} to the smooth
problem \eqref{eq:omega:Y}, making the smooth convex optimization
tools applicable.

\begin{omegaSmooth}\label{theorem:omegaSmooth}
The function $\omega( Y)$ is convex and continuously differentiable
with
\begin{equation}\label{eq:omega:gradient}
  \omega'(Y)=-\max( \mathbf u - Y \mathbf e, \mathbf 0) \mathbf
  e^{\rm T}.
\end{equation}
In addition, $\omega'(Y)$ is Lipschitz continuous with constant
$g^2$, i.e.,
\begin{equation}\label{eq:omega:Lip}
  \|\omega'(Y_1) - \omega'(Y_2) \|_F \leq g^2 \|Y_1 -Y_2\|_F,
 \,\,\,\,  \forall \,\, Y_1, Y_2 \in \mathbb{R}^{p \times g}.
\end{equation}
\end{omegaSmooth}

To prove Theorem~\ref{theorem:omegaSmooth}, we first present two
technical lemmas. The first lemma is related to the optimal value
function~\cite{Bonnans:1998,Danskin:1967}, and it was used in a
recent study~\cite{Ying:2009:nips} on infinite kernel learning.
\begin{MinMaxOmega}\label{lemma:minmaxOmega}
\cite[Theorem~4.1]{Bonnans:1998} Let $X$ be a metric space and $U$ be
a normed space. Suppose that for all $\mathbf x \in X$, the function
$\psi(\mathbf x, \cdot)$ is differentiable and that $\psi(\mathbf x,
Y)$ and $D_{Y} \psi(\mathbf x, Y)$ (the partial derivative of
$\psi(\mathbf x, Y)$ with respect to $Y$) are continuous on $X \times
U$. Let $\Phi$ be a compact subset of $X$. Define the optimal value
function as $\varphi(Y)=\inf_{\mathbf x \in \Phi} \psi(\mathbf x,
Y)$. The optimal value function $\varphi(Y)$ is directionally
differentiable. In addition, if $\forall Y \in U$, $\psi(\cdot, Y)$
has a unique minimizer $\mathbf x (Y)$ over $\Phi$, then $\varphi(Y)$
is differentiable at $Y$ and the gradient of $\varphi(Y)$ is given by
$\varphi'(Y)=D_{Y} \psi(\mathbf x(Y), Y)$.

\end{MinMaxOmega}

The second lemma shows that the operator $\mathbf y= \max(\mathbf
x,\mathbf 0)$ is non-expansive.
\begin{convexLip}\label{lemma:convexLip}
$\forall \mathbf x, \mathbf y \in \mathbb{R}^p$, we have
$\|\max(\mathbf x,\mathbf 0)-\max(\mathbf y,\mathbf 0)\| \leq
\|\mathbf x-\mathbf y\|$.
\end{convexLip}
\begin{proof}
The results follows since $|\max(x, 0)-\max( y, 0)| \leq | x- y|$,
$\forall x, y \in \mathbb{R}$. \end{proof}


\textbf{Proof of Theorem~\ref{theorem:omegaSmooth}: }  To prove the
differentiability of $\omega(Y)$, we apply
Lemma~\ref{lemma:minmaxOmega} with $X=\mathbb{R}^p$,
$U=\mathbb{R}^{p \times g}$ and $\Phi=\{\mathbf x \in X: \mathbf u+
\lambda_2 \sum w_i \mathbf e  \geq \mathbf x \geq \mathbf 0\}$. It
is easy to verify that 1) $\psi(\mathbf x, \cdot)$ is
differentiable; 2) $\psi(\mathbf x, Y)$ and $D_{Y} \psi(\mathbf x,
Y)=\mathbf x \mathbf e^{\rm T}$  are continuous on $X \times U$; 3)
$\Phi$ be a compact subset of $X$; and 4) $\forall Y \in U$,
$\psi(\mathbf x, Y)$ has a unique minimizer $\mathbf x (Y)=\max(
\mathbf u - Y \mathbf e, \mathbf 0)$ over $\Phi$. Note that, the
last result follows from $\mathbf u >0$ and $\mathbf u - Y \mathbf e
\leq \mathbf u + \lambda_2 \sum w_i \mathbf e$, where the latter
inequality utilizes $\|Y^i\| \le \lambda_2 w_i$; and this indicates
that $\mathbf x (Y)=\max( \mathbf u - Y \mathbf e, \mathbf 0)=\arg
\min_{\mathbf x} \psi(\mathbf x, Y)=\arg \min_{\mathbf x \in \Phi}
\psi(\mathbf x, Y)$. It follows from Lemma~\ref{lemma:minmaxOmega}
that
$$\varphi(Y)=\inf_{\mathbf x \in
\Phi} \psi(\mathbf x, Y)=\psi(\max( \mathbf u - Y \mathbf e, \mathbf
0), Y)$$ is differentiable with $\varphi'(Y)=\max( \mathbf u - Y
\mathbf e, \mathbf 0) \mathbf e^{\rm T}$.

In \eqref{eq:min:max}, $\psi(\mathbf x, Y)$ is convex in $\mathbf x$
and concave in $Y$, and the constraint sets are closed convex for
both $\mathbf x$ and $Y$, thus the existence of the saddle point is
guaranteed by the well-known von Neumann
Lemma~\cite[Chapter~5.1]{NEMIROVSKI:1994:ID6}. As a result,
$$\varphi(Y)=\inf_{\mathbf x \in \Phi} \psi(\mathbf x, Y)=\psi(\max(
\mathbf u - Y \mathbf e, \mathbf 0), Y)$$ is concave, and
$\omega(Y)=-\varphi(Y)$ is convex. For any $Y_1, Y_2$, we have
\begin{equation}
\begin{aligned}
  \|\omega'(Y_1) - \omega'(Y_2) \|_F  = &\|\max( \mathbf u - Y_1 \mathbf e, \mathbf 0) \mathbf
  e^{\rm T}-\max( \mathbf u - Y_2 \mathbf e, \mathbf 0) \mathbf e^{\rm T}\|_F \\
   \leq & \|\mathbf e\| \times \|\max( \mathbf u - Y_1 \mathbf e, \mathbf 0)-\max( \mathbf u - Y_2 \mathbf e, \mathbf
  0)\| \\
   \leq & \|\mathbf e\| \times \| (Y_1- Y_2) \mathbf e\| \\
   \leq & g^2 \|Y_1- Y_2\|_F,
\end{aligned}
\end{equation}
where the second inequality follows from Lemma~\ref{lemma:convexLip}.
We prove \eqref {eq:omega:Lip}. \hfill $\Box$


From Theorem~\ref{theorem:omegaSmooth}, the problem
in~\eqref{eq:omega:Y} is a constrained smooth convex optimization
problem, and existing solvers for constrained smooth convex
optimization can be applied. In this paper, we employ the
accelerated gradient descent to solve \eqref{eq:omega:Y}, due to its
fast convergence property. Note that, the Euclidean projection onto
the set $\Omega$ can be computed in the closed form. We would like
to emphasize here that, the problem~\eqref{eq:omega:Y} may have a
much smaller size than~\eqref{eq:problem:m:y}.

\subsection{Computing the Duality Gap}

We show how to estimate the duality gap of the min-max problem
\eqref{eq:min:max}, which can be used to check the quality of the
solution and determine the convergence of the algorithm.


For any given approximate solution $\tilde Y \in \Omega$ for
$\omega(Y)$, we can construct the approximate solution $\mathbf
{\tilde x}=\max(\mathbf u - \tilde Y \mathbf e, \mathbf 0)$ for
$h_{\lambda_2}(\mathbf x)$. The duality gap for the min-max problem
\eqref{eq:min:max} at the point $(\mathbf {\tilde x},\tilde Y)$ can
be computed as:
\begin{equation}\label{eq:gap:Y}
    \mbox{gap}(\tilde Y)= \max_{Y \in \Omega} \psi(\tilde {\mathbf x},
    Y) -\min_{\mathbf x \in \mathbb{R}^p: \mathbf x \geq \mathbf 0} \psi(\mathbf x, \tilde
    Y).
\end{equation}
The main result of this subsection is summarized in the following theorem:
\begin{dualityGap}\label{theorem:duality:gap}
Let ${\rm gap}(\tilde Y)$ be the duality gap defined in
\eqref{eq:gap:Y}. Then, the following holds:
\begin{equation}\label{eq:gap:Y:result}
    {\rm gap}(\tilde Y)= \lambda_2 \sum_{i=1}^g
    ( w_i\|\tilde {\mathbf x}_{G_i}\|-\langle \tilde {\mathbf x}_{G_i}, \tilde {Y}_{G_i}^i
    \rangle).
\end{equation}
In addition, we have
\begin{equation}\label{eq:gap:1}
   \omega(\tilde Y)- \omega(Y^*) \leq {\rm gap}(\tilde Y),
\end{equation}
\begin{equation}\label{eq:gap:2}
   h(\tilde {\mathbf
    x})- h(\mathbf x^*) \leq  {\rm gap}(\tilde Y).
\end{equation}
\end{dualityGap}

\begin{proof}
Denote $(\mathbf x^*, Y^*)$ as the optimal solution to the min-max
problem \eqref{eq:min:max}. From
\eqref{eq:h:x:max}-\eqref{eq:omega:Y}, we have
\begin{equation}\label{eq:inequality:1}
    -\omega(\tilde Y)= \psi(\tilde {\mathbf x}, \tilde Y) =\min_{\mathbf x \in \mathbb{R}^p: \mathbf x \geq \mathbf 0} \psi(\mathbf x, \tilde Y) \leq
   \psi(\mathbf x^*, \tilde Y),
\end{equation}
\begin{equation}\label{eq:inequality:2}
    \psi(\mathbf x^*, \tilde Y) \leq \max_{Y \in \Omega} \psi(\mathbf x^*, Y) = \psi(\mathbf x^*, Y^*) = - \omega(Y^*),
\end{equation}
\begin{equation}\label{eq:inequality:3}
    h_{\lambda_2}(\mathbf x^*) = \psi(\mathbf x^*, Y^*) = \min_{\mathbf x \in \mathbb{R}^p: \mathbf x \geq \mathbf 0}  \psi(\mathbf x, Y^*) \leq \psi(\tilde {\mathbf x}, Y^*),
\end{equation}
\begin{equation}\label{eq:inequality:4}
    \psi(\tilde {\mathbf x}, Y^*) \leq \max_{Y \in \Omega} \psi(\tilde {\mathbf x}, Y) = h_{\lambda_2}(\tilde {\mathbf
    x}).
\end{equation}
Incorporating \eqref{eq:h:x:positive},
\eqref{eq:inequality:1}-\eqref{eq:inequality:4}, we prove
\eqref{eq:gap:Y:result}-\eqref{eq:gap:2}. \end{proof} In our
experiments, we terminate the algorithm when the estimated duality
gap is less than $10^{-10}$.

\section{Experiments}

We have conducted experiments to evaluate the efficiency of the
proposed algorithm using the breast cancer gene expression data
set~\cite{VV:2002}, which consists of 8,141 genes in 295 breast
cancer tumors (78 metastatic and 217 non-metastatic). For the sake of
analyzing microarrays in terms of biologically meaningful gene sets,
different approaches have been used to organize the genes into
(overlapping) gene sets. In our experiments, we
follow~\cite{Jacob:2009} and employ the following two approaches for
generating the overlapping gene sets (groups):
pathways~\cite{Subramanian:2005} and edges~\cite{Chuang:2007}. For
pathways, the canonical pathways from the Molecular Signatures
Database (MSigDB)~\cite{Subramanian:2005} are used. It contains 639
groups of genes, of which 637 groups involve the genes in our study.
The statistics of the 637 gene groups are summarized as follows: the
average number of genes in each group is 23.7, the largest gene group
has 213 genes, and 3,510 genes appear in these 637 groups with an
average appearance frequency of about 4. For edges, the network built
in~\cite{Chuang:2007} will be used, and we follow~\cite{Jacob:2009}
to extract 42,594 edges from the network, leading to 42,594
overlapping gene sets of size 2. All 8,141 genes appear in the 42,594
groups with an average appearance frequency of about 10.

\begin{figure}[!t]
\centering
  \includegraphics[width=1.8in]{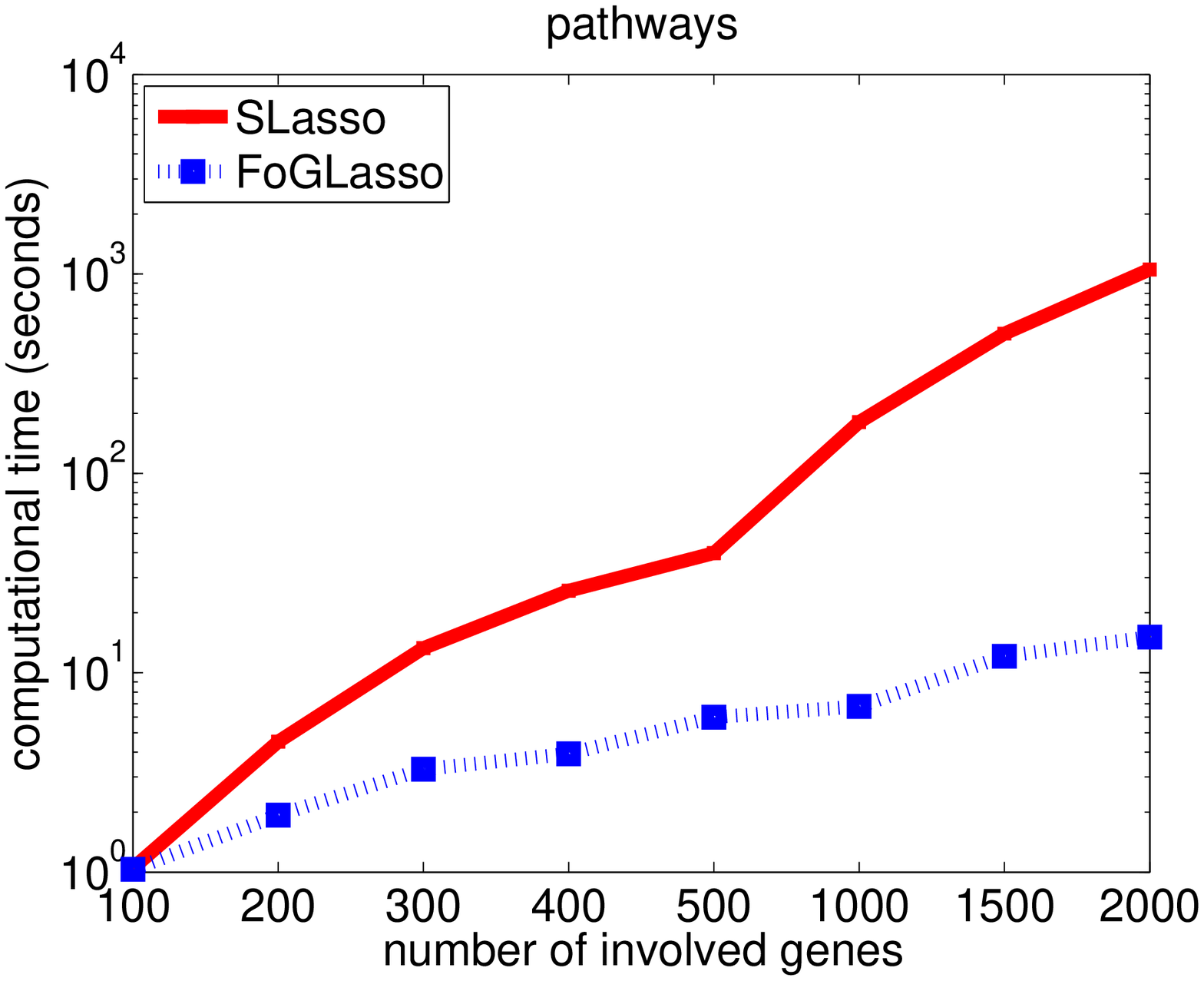}
  \includegraphics[width=1.8in]{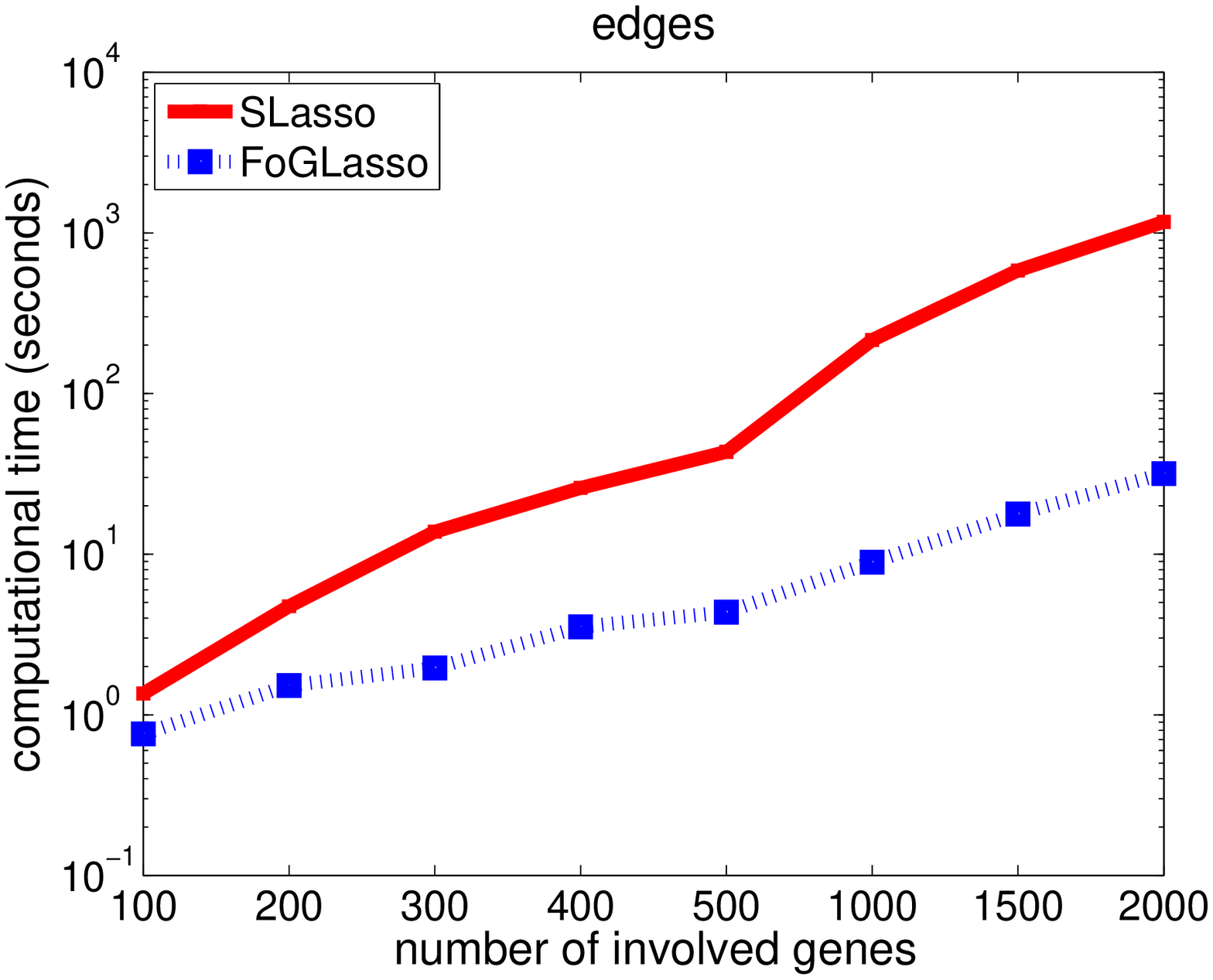} \\[-0.3cm]
\caption{
Comparison of SLasso~\cite{Jenatton:2009} and our proposed FoGLasso
algorithm in terms of computational time (in seconds) when different
numbers of genes (variables) are involved. The computation time is
reported in the logarithmic scale.} \label{fig:efficiency}
\end{figure}

\vspace{0.1in}

\noindent \textbf{Efficiency of the Proposed FoGLasso }  We compare
our proposed FoGLasso with the SLasso algorithm developed by Jenatton
et al.~\cite{Jenatton:2009} for solving~\eqref{eq:problem} with the
least squares loss $l(\mathbf x)=\frac{1}{2} \|A \mathbf x - \mathbf
b\|^2$. The experimental settings are as follows: we set
$w_i=\sqrt{|G_i|}$, and $\lambda_1=\lambda_2=\rho \times
\lambda_1^{\max}$, where $|G_i|$ denotes the size of the $i$-th group
$G_i$, $\lambda_1^{\max}=\|A^{\rm T} \mathbf b\|_{\infty}$ (the zero
point is a solution to~\eqref{eq:problem} if $\lambda_1
\geq\lambda_1^{\max}$), and $\rho$ is chosen from the set $\{5 \times
10^{-1}, 2 \times 10^{-1}, 1\times 10^{-1},5\times 10^{-2}, 2 \times
10^{-2}, 1 \times 10^{-2}, 5 \times 10^{-3}, 2 \times 10^{-3},
1\times 10^{-3}\}$. For a given $\rho$, we first run SLasso, and then
run our proposed FoGLasso until it achieves an objective function
value smaller than or equal to that of SLasso. For both SLasso and
FoGLasso, we apply the ``warm'' start technique, i.e., using the
solution corresponding to the larger regularization parameter as the
``warm'' start for the smaller one. We vary the number of genes
involved, and report the total computational time (seconds) including
all nine regularization parameters in Figure~\ref{fig:efficiency} and
Table~\ref{table:efficiency}. We can observe that, 1) our proposed
FoGLasso is much more efficient than SLasso; 2) the advantage of
FoGLasso over SLasso in efficiency grows with the increasing number
of genes (variables). For example, with the grouping by pathways,
FoGLasso is about 25 and 70 times faster than SLasso for 1000 and
2000 genes (variables), respectively; and 3) the efficiency on edges is inferior
to that on pathways, due to the larger number of overlapping groups.
These results verify the efficiency of the proposed FoGLasso
algorithm based on the efficient procedure for computing the proximal
operator presented in Section~\ref{s:proximal:operator}.


\begin{table}
\caption{Scalability study of the proposed FoGLasso algorithm under
different numbers ($p$) of genes involved. The reported results are
the total computational time (seconds) including all nine
regularization parameter values.} \label{table:efficiency} \centering
   \begin{tabular}{lllllll}
    \hline
    $p$      & 3000 & 4000 & 5000 & 6000 & 7000 & 8141 \\
    pathways &    37.6 &   48.3 &   62.5 &   68.7 &   86.2 &   99.7 \\
    edges    &    58.8 &   84.8 &  102.7 &  140.8 &  173.3 &  247.8 \\
    \hline
   \end{tabular}
\end{table}

\vspace{0.1in}

\noindent \textbf{Computation of the Proximal Operator }  In this
experiment, we run FoGLasso on the breast cancer data set using all
8,141 genes. We terminate FoGLasso if the change of the objective
function value is less than $10^{-5}$. We use the 42,594 edges to
generate the overlapping groups. We obtain similar results for the
637 groups based on pathways. We set $\rho=0.01$. The results are
shown in Figure~\ref{fig:projection}. The left plot shows that the
objective function value decreases rapidly in the proposed FoGLasso.
In the middle plot, we report the percentage of the identified zero
groups by applying Lemma~\ref{corollary:InducedSparsity}. Our
experimental result shows that, 1) after 16 iterations, 50\% of the
zero groups are correctly identified; and 2) after 50 iterations,
80\% of the zero groups are identified. Therefore, with
Lemma~\ref{corollary:InducedSparsity}, we can significantly reduce
the problem size of the subsequent dual reformulation (see
Section~\ref{ss:proximal:operator}). In the right plot of
Figure~\ref{fig:projection}, we present the number of inner
iterations for solving the proximal operator via the dual
reformulation. We attribute the decreasing number of inner iterations
to 1) the size of the reduced problem is decreasing when many zero
groups are identified (see the middle plot); and 2) in solving the
dual reformulation, we can apply the $Y$ computed in the previous
iteration as the ``warm'' start for computing the proximal operator
in the next iteration.

\begin{figure}
\centering
  \includegraphics[width=1.4in]{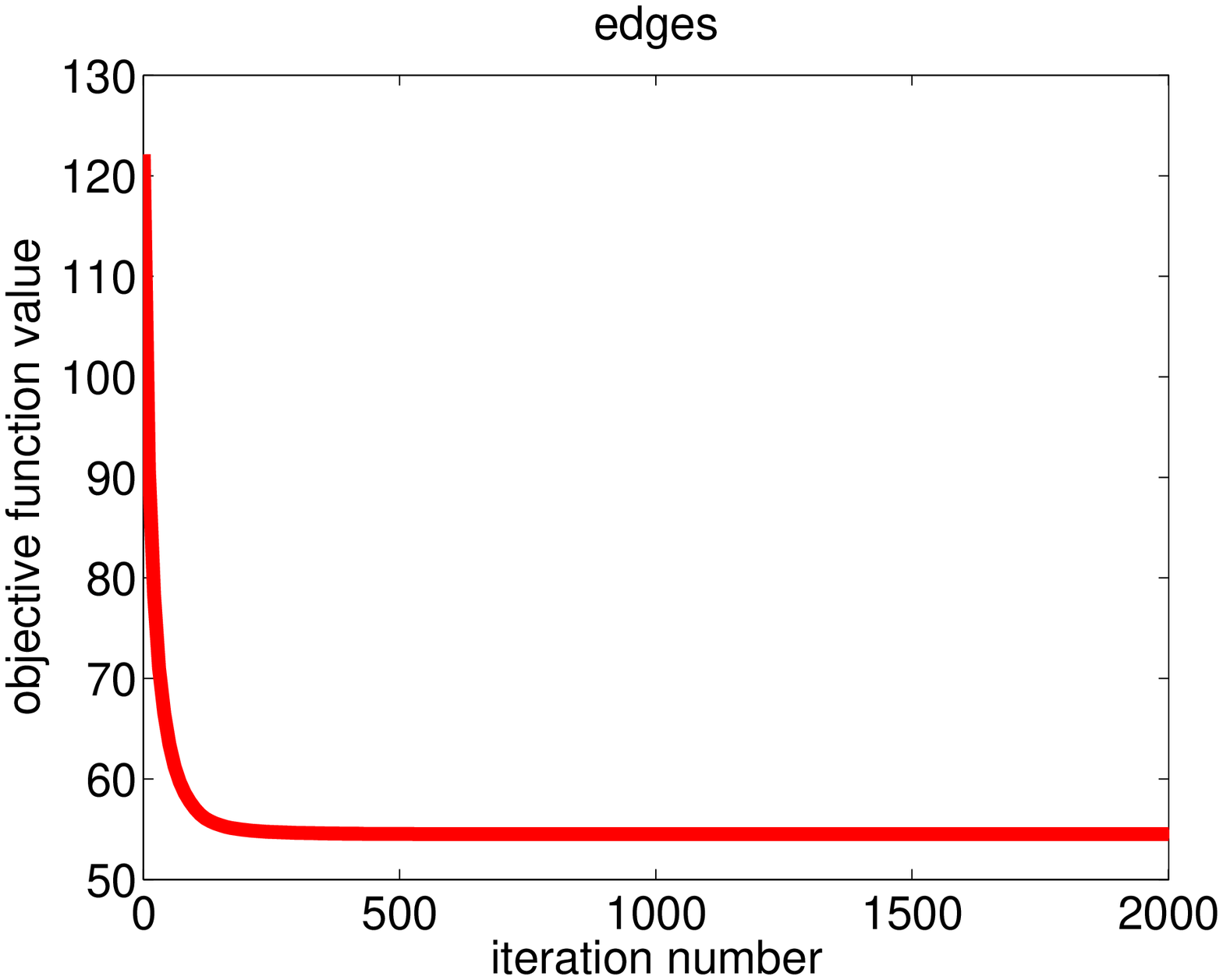}
  \includegraphics[width=1.4in]{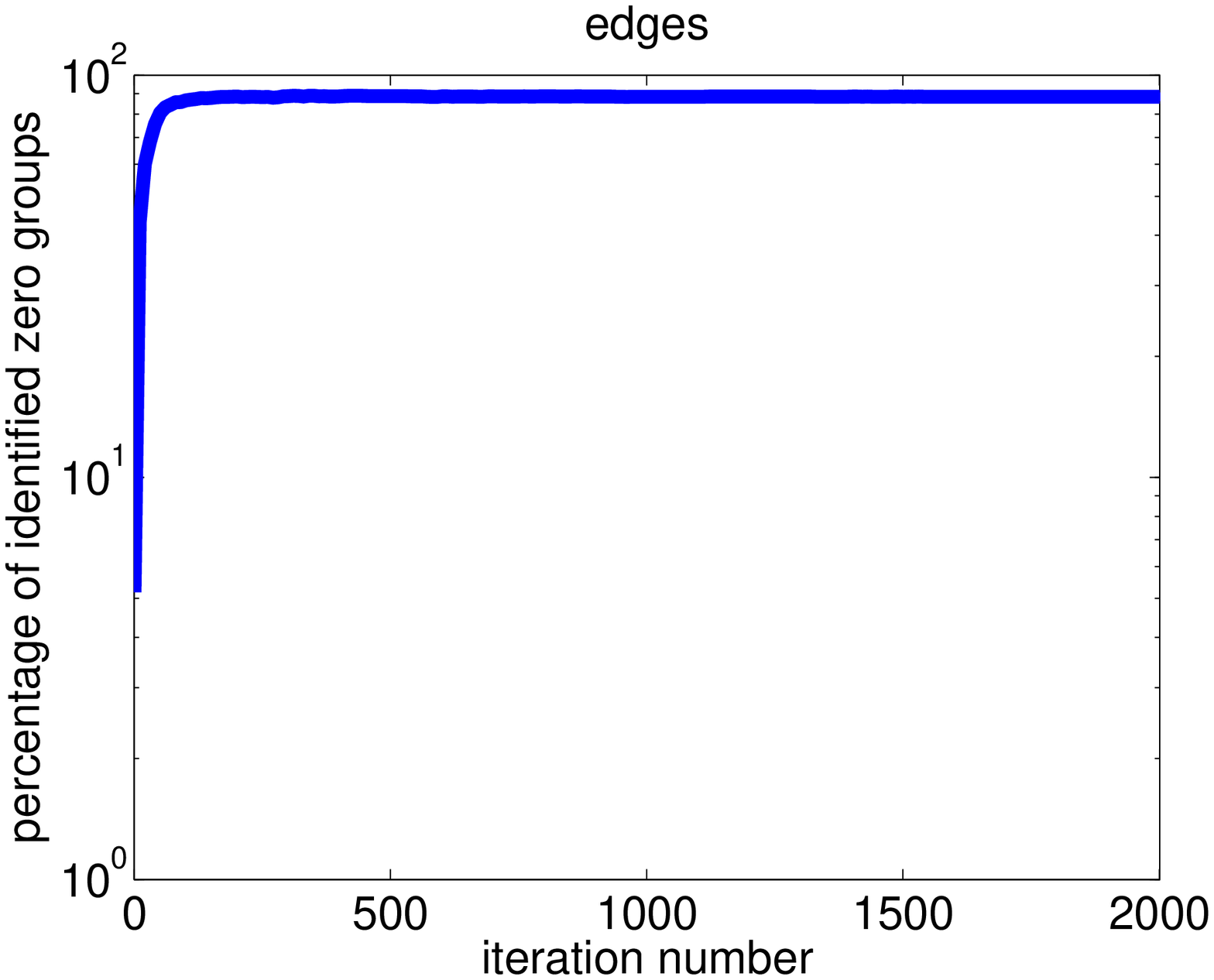}
  \includegraphics[width=1.4in]{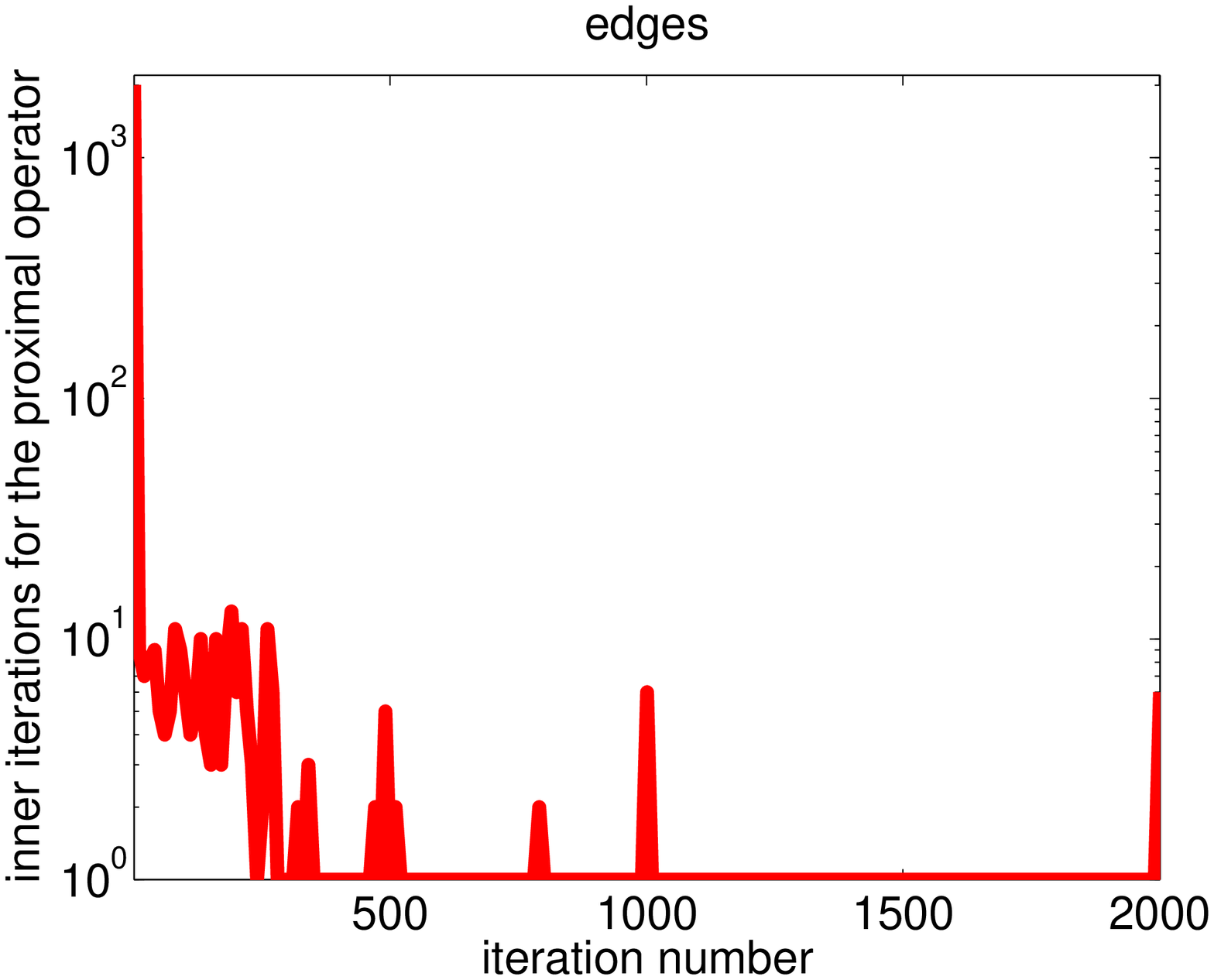} \\[-0.25cm]
\caption{
Performance of the computation of the proximal operator in FoGLasso.
The left plot shows the objective function value during the FoGLasso
iteration. The middle plot shows the percentage of the identified
zero groups by applying Lemma~\ref{corollary:InducedSparsity}.
The right plot shows the number of inner iterations for achieving the
duality gap less than $10^{-10}$ when one solves the proximal operator
via the dual reformulation (see Section 3.2).} \label{fig:projection}
\end{figure}

\vspace{0.1in}

\noindent \textbf{Classification Performance }  We compare the
classification performance of the overlapping group Lasso with Lasso.
We use 60\% samples for training and the rest 40\% for testing. To
deal with the imbalance of the positive and negative samples, we make
use of the balanced error rate~\cite{Guyon:2004}, which is defined as
the average error of two classes. We report the results averaged over
10 runs in Figure~\ref{fig:performance}. Our results show that: 1)
with the overlapping pathways, overlapping Lasso and Lasso achieve
comparable classification performance; 2) with the overlapping edges,
overlapping Lasso outperforms Lasso; and 3) the performance based on
edges is better than that based on the pathways in our experiment.

\begin{figure}
\centering
  \includegraphics[width=1.8in]{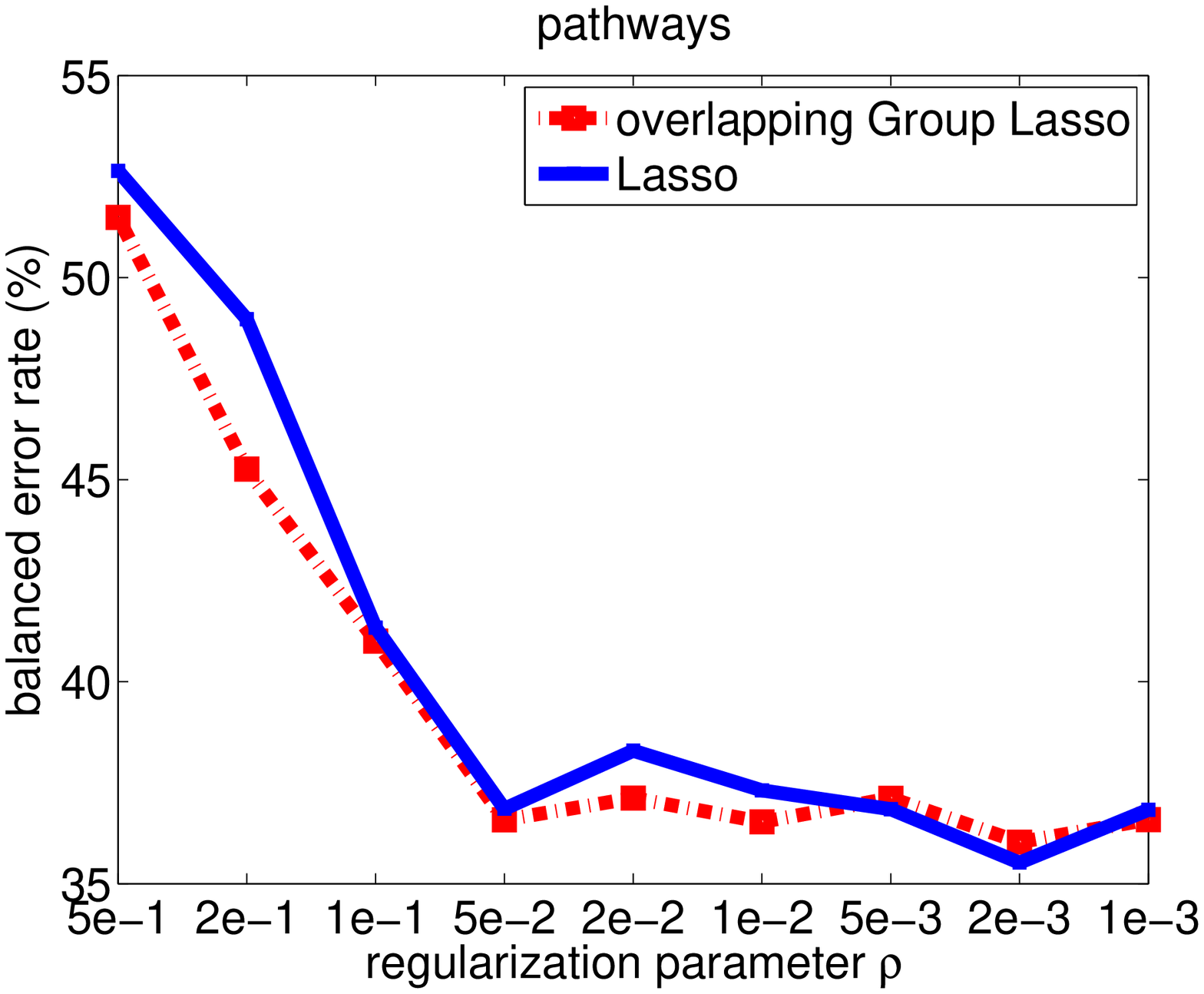}
  \includegraphics[width=1.8in]{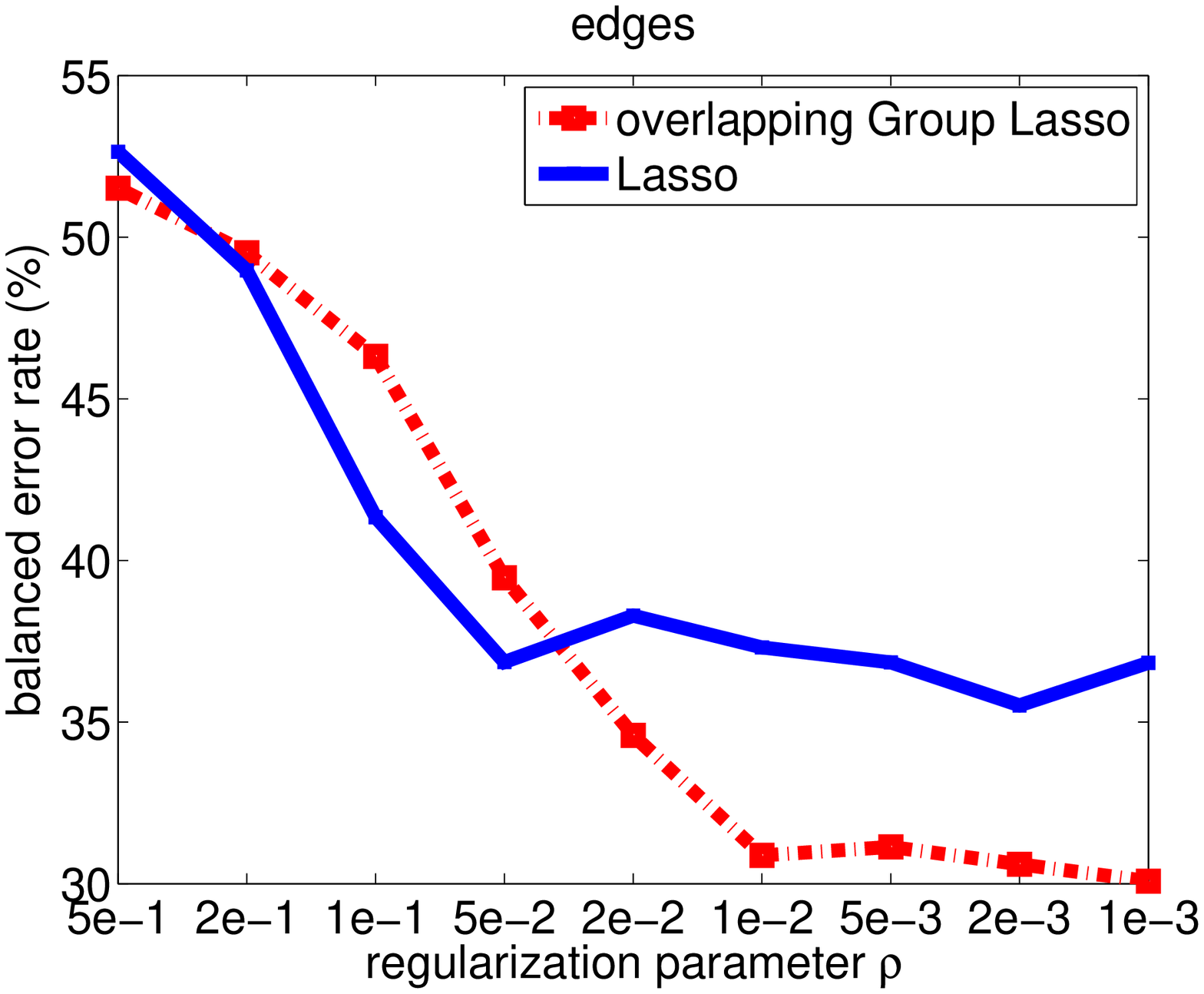} \\[-0.25cm]
\caption{
Comparison of overlapping group Lasso and Lasso in terms of the
balanced error rate. The left plot shows the classification performance with
overlapping pathways; and the right plot shows the result with the
overlapping edges.} \label{fig:performance}
\end{figure}

\section{Conclusion}

In this paper, we consider the efficient optimization of the
overlapping group Lasso penalized problem based on the accelerated
gradient descent method. We reveal several key properties of the
proximal operator associated with the overlapping group Lasso, and
compute the proximal operator via solving the smooth and convex dual
problem. Numerical experiments on the breast cancer data set
demonstrate the efficiency of the proposed algorithm. Our
experimental results also show the benefit of the overlapping group
Lasso in comparison with Lasso. In the future, we plan to apply the
proposed algorithm to other real-world applications involving
overlapping groups.

{\small
\bibliographystyle{plain}
\bibliography{agmeep,cgLasso,fusedLasso,manual,mtl}
}

\end{document}